\documentclass[letterpaper, 10 pt, conference]{ieeeconf}  

\IEEEoverridecommandlockouts                              

\usepackage{amsmath, amssymb, mathtools}

\usepackage{amsthm}
\usepackage{bm}
\usepackage{color}
\usepackage{graphicx}
\usepackage{subcaption}
\usepackage{booktabs}
\usepackage[
    style=numeric,
    sorting=none,        
    sortcites=true,      
    backend=biber, 
    maxbibnames=99,      
    natbib=true,
    doi=false,   
    url=false,   
    isbn=false,  
    eprint=false, 
    giveninits=true
]{biblatex}

\newtheorem{theorem}{Theorem}[section]

\newtheorem{proposition}[theorem]{Proposition}

\newtheorem{definition}[theorem]{Definition}
\newtheorem{remark}[theorem]{Remark}
\newtheorem{assumption}[theorem]{Assumption}

\usepackage{etoolbox}

\newbool{anonymous}
\setbool{anonymous}{false} 

\newcommand{\anon}[2]{\ifbool{anonymous}{#1}{#2}}
\newcommand{\rev}[1]{{#1}}

\begin{document}

\def\BibTeX{{\rm B\kern-.05em{\sc i\kern-.025em b}\kern-.08em
    T\kern-.1667em\lower.7ex\hbox{E}\kern-.125emX}}
\markboth{\journalname, VOL. XX, NO. XX, XXXX 2017}
{Author \MakeLowercase{\textit{et al.}}: Preparation of Papers for IEEE Control Systems Letters (August 2022)}

\title{\bf \rev{Variable Aerodynamic Damping Actuation via Co-Contraction:\\
A Structural Analogy with Variable Stiffness Actuation}}

\author{
\anon{Authors Removed}{Antonio Franchi$^{1,2}$}
\anon{}{\thanks{$^1$Robotics and Mechatronics, Electrical Engineering,  Mathematics, and Computer Science Faculty, University of Twente, The Netherlands. 
}
\thanks{$^2$Department of Computer, Control and Management Engineering, Sapienza University of Rome, Rome, Italy. {\tt\footnotesize schol@r-franchi.eu}}
\thanks{The work was partially funded by the European Commission Horizon Europe Framework under project Autoassess (101120732)}}
}

\maketitle


\begin{abstract}
\rev{This work identifies a passive aerodynamic damping effect induced by co-contraction in antagonistic redundant propulsion. Complementing prior work on aerodynamic promptness, which addressed active wrench-rate authority along constant-wrench fibers, we study the passive side: the local derivative of aerodynamic force with respect to air-relative velocity at a trim. This derivative defines an incremental aerodynamic damping coefficient. We prove that it increases monotonically along constant-force fibers under a mild aerodynamic hardening condition, and derive this property from a first-order Blade Element Theory model exposing the relevant speed--inflow coupling. The resulting mechanism, Variable Aerodynamic Damping Actuation (VADA), is formulated as an antagonistic aerodynamic actuation module and allocation principle, structurally analogous to variable-stiffness actuation at the level of fiber motions and incremental impedance modulation. An impedance-form interpretation clarifies common- and differential-mode roles, while a propeller-data-based assessment using the UIUC Propeller Database shows that the identified damping has practical small-UAV magnitude, is comparable to ordinary low-speed body-drag damping, and depends strongly on low-advance-ratio thrust sensitivity.}
\end{abstract}

\section{Introduction}
\label{sec:intro}
Actuation redundancy is pervasive in multirotor aerial vehicles, including standard hexarotors/octorotors and fully actuated platforms~\cite{park2018_tmech_odar,bodie2020_iser_volirox,Ryll2015TCST}, where multiple rotor-speed choices can realize the same commanded wrench.
This freedom is typically resolved by control-allocation strategies that minimize an effort proxy, such as least-norm or QP-type allocators~\cite{Bodson2002JGCD,Brescianini2016ICRA}.
While such choices are well aligned with endurance, they do not explicitly optimize dynamic \rev{control readiness and passive envronmental response}, which can become critical in \rev{many real-world situations}.

The rejection of aggressive disturbances and the execution of agile maneuvers are limited not only by achievable steady-state wrench magnitudes, but also by motor torque and rotor-acceleration constraints that bound wrench-rate generation~\cite{Romero2022TRO_MPCC,kaufmann2020_rss_acrobatics,saviolo2022_ral_pitcn,Bicego2020JIRS}.
Motivated by manipulability in robotics~\cite{Yoshikawa1985IJRR},~\cite{Franchi2026ICUAS} introduced a geometric, fiber-based interpretation of this phenomenon for multirotors: moving along a constant-wrench fiber injects internal aerodynamic loading, or aerodynamic co-contraction, thereby increasing the local fiber density and the associated aerodynamic promptness.

\rev{The present work addresses the passive counterpart of aerodynamic promptness. Previous work studied how co-contraction changes the speed-to-wrench differential and hence active wrench-rate authority. Here we study the derivative of aerodynamic force with respect to air-relative velocity at a trim. We show that this is the propeller-driven analogue of the VSA stiffness derivative with respect to external joint deflection: co-contraction shapes a local force--velocity slope while preserving the commanded force.}
In antagonistic Variable Stiffness Actuators (VSAs), co-contraction increases active authority and shapes the passive mechanical response via impedance regulation~\cite{Hogan1984ACC,Burdet2001Nature,Franklin2003JNP,Gribble2003JNP}.
\rev{The  question considered here is therefore whether aerodynamic co-contraction can tune a passive incremental input--output property while keeping the commanded force fixed.}
We answer positively this question under a quasi-steady thrust model in which rotor thrust depends on rotor speed and axial inflow~\cite{Bristeau2009ECC,Hoffmann2007QuadrotorAIAA}.
Specifically, we define incremental aerodynamic damping as the local map from air-relative velocity perturbations to aerodynamic force perturbations, prove that co-contraction along constant-force fibers monotonically increases it under aerodynamic hardening, and verify this condition with a first-order Blade Element Theory model.
\rev{We then estimate the same force--velocity derivative from UIUC propeller data, connecting the theoretical coefficient to measured low-advance-ratio thrust curves and to ordinary body-drag scales.}
\rev{The resulting Variable Aerodynamic Damping Actuation (VADA) mechanism is interpreted as a local structural analogy with VSA stiffness modulation at the level of fiber motions and incremental impedance modulation.}

\smallskip
\noindent\textbf{Contributions.}
Compared to existing uses of manipulability in aerial robotics that focus on attached manipulators or static design indices~\cite{Ruggiero2018RAL,Hamandi2021IJRR} and to the recent introduction of aerodynamic promptness~\cite{Franchi2026ICUAS}, this work contributes:
(i) \rev{a trim-based definition of incremental aerodynamic damping modulation, distinct from active aerodynamic promptness,}
(ii) a proof that co-contraction along constant-force fibers increases the damping coefficient under hardening,
(iii) \rev{the formulation of Variable Aerodynamic Damping Actuation as an antagonistic aerodynamic actuation module and allocation principle,}
(iv) an impedance-form interpretation of common- and differential-mode actuation in the affine-inflow regime, and
(v) \rev{a UIUC-propeller-data-based scale assessment quantifying damping magnitude, power cost, propeller-family dependence, and comparison with ordinary body-drag damping.}

\section{Passive Stiffness and Active Promptness in Antagonistic VSAs}
\label{sec:vsa}
This section recalls a standard antagonistic Variable Stiffness Actuator (VSA) model~\cite{Wolf-vsa-2016} and makes explicit a basic but useful fact: along constant-torque fibers, internal co-contraction simultaneously increases (i) passive stiffness---the incremental torque response to external deflections---and (ii) active promptness---the local fiber density (manipulability) induced by the actuator task map, in the sense of~\cite{Franchi2026ICUAS}.
Fig.~\ref{fig:vsa_schematic} summarizes the variables and tendon routing.

\begin{figure}[t]
    \centering
    \begin{subfigure}[t]{0.44\columnwidth}
        \centering
        \includegraphics[width=\linewidth]{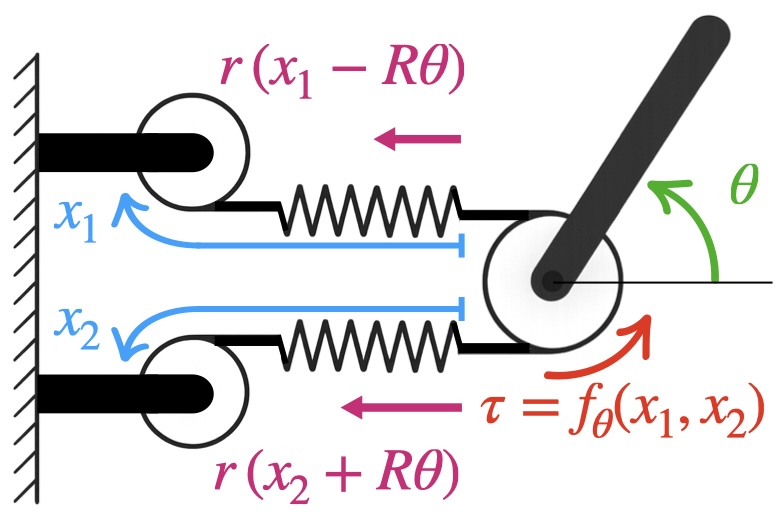}
        \caption{Antagonistic VSA. Two motors generate joint torque and locally tune stiffness via co-contraction.}
        \label{fig:vsa_schematic}
    \end{subfigure}
    \hfill
    \begin{subfigure}[t]{0.54\columnwidth}
        \centering
        \includegraphics[width=\linewidth]{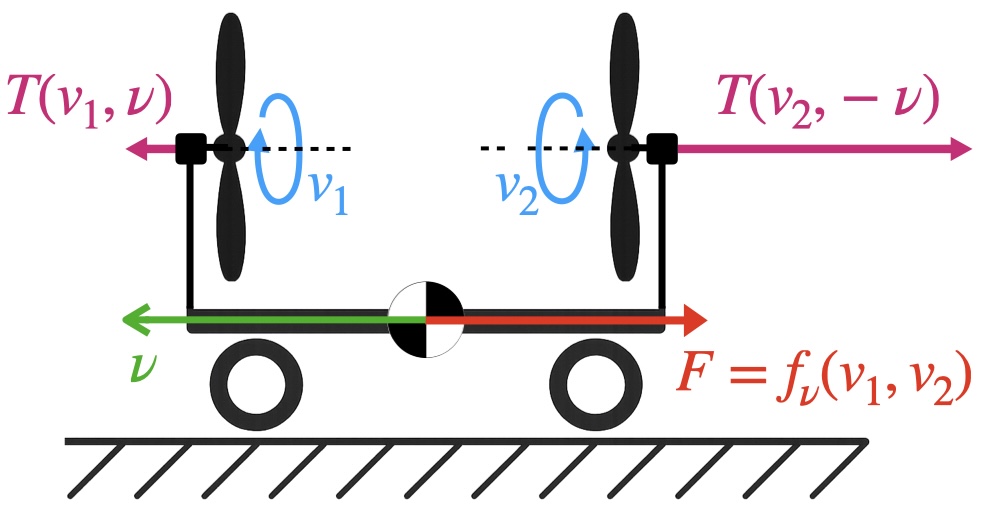}
        \caption{Antagonistic dual-rotor module. Two rotors generate net force and locally tune aerodynamic damping via co-contraction.}
        \label{fig:dualrotor_schematic}
    \end{subfigure}
    \caption{Antagonistic actuation mechanisms. Both admit fiber motions corresponding to internal co-contraction: stiffness modulation in the VSA and aerodynamic damping modulation in the dual-rotor module. \rev{The aerodynamic panel represents the idealized non-interfering inflow case used to isolate the mechanism.}}
    \label{fig:vsa_dualrotor_combined}
\end{figure}

\subsubsection*{Actuation model} Consider a single revolute joint with deflection $\theta\in\mathbb{R}$ driven by an antagonistic tendon mechanism. Let the actuator configuration be $\bm{x}=(x_1,x_2)\in\mathcal{X}\subset\mathbb{R}^2_{>0}$, where $x_1,x_2$ are the motor-controlled tendon displacements \rev{imposed when $\theta=0$}, and let $R>0$ be the pulley radius. \rev{The tendon extensions are therefore $x_1-R\theta$ and $x_2+R\theta$.} The tendon force--extension law, common to both tendons, is modeled by $r:\mathbb{R}\to\mathbb{R}_{\ge 0}$.

\begin{assumption}[Strict hardening and local admissibility]
\label{ass:vsa_hardening}
The tendon \rev{force} law satisfies $r\in\mathcal{C}^2$ and, for all $x>0$, $r'(x)>0$ and $r''(x)>0$.
We restrict attention to operating points and sufficiently small deflections such that $x_1-R\theta>0$ and $x_2+R\theta>0$ in the neighborhood of interest.
\end{assumption}
Under this model, the net elastic torque exerted on the joint is $\tau:\mathcal{X}\times \mathbb{R}\to \mathbb{R}$ defined by
\begin{equation}
\label{eq:vsa_tau}
(\bm{x},\theta)\mapsto \tau(\bm{x},\theta)
\coloneqq
R\Big(r(x_1-R\theta)-r(x_2+R\theta)\Big).
\end{equation}

\subsection{Stiffness and promptness at $\theta=0$}
\label{subsec:vsa_theta0}
At the nominal operating point $\theta=0$, define the task map $f:\mathcal{X}\to\mathbb{R}$ as
\begin{equation}
\label{eq:vsa_task_map}
\bm{x}\mapsto f(\bm{x})
\coloneqq \tau(\bm{x},0)
= R\big(r(x_1)-r(x_2)\big).
\end{equation}
The fiber of a desired torque $\bar\tau\in\mathbb{R}$ is the level set
\begin{equation}
\label{eq:vsa_fiber}
\mathcal{F}_{\bar\tau}\coloneqq f^{-1}(\bar\tau)=
\{\bm{x}\in\mathcal{X}\mid f(\bm{x})=\bar\tau\},
\end{equation}
consistent with the fiber viewpoint used in~\cite{Franchi2026ICUAS}.

\subsubsection*{Passive stiffness}
The passive joint stiffness $\sigma:\mathcal{X}\to\mathbb{R}_{>0}$ is the incremental resistance to external angular deflections at $\theta=0$:
\begin{equation}
\label{eq:vsa_def_stiffness}
\bm{x}\mapsto \sigma(\bm{x})
\coloneqq
-\left.\tfrac{\partial \tau}{\partial\theta}\right|_{\theta=0}
=
R^2\big(r'(x_1)+r'(x_2)\big).
\end{equation}

\subsubsection*{Active promptness (fiber density)}
The active torque promptness $\rho:\mathcal{X}\to \mathbb{R}_{\geq 0}$ is the fiber-density (manipulability) under the Euclidean metric on actuator rates~\cite{Franchi2026ICUAS}:
\begin{equation}
\bm{x}\mapsto\rho(\bm{x})
\coloneqq
\sqrt{\det(J_f(\bm{x})\;J_f^\top(\bm{x}))},
\label{eq:torque_promptness}
\end{equation}
where $J_f(\bm{x})\coloneqq\nabla f(\bm{x})$.
For the scalar task \eqref{eq:vsa_task_map},
\begin{equation}
\label{eq:vsa_Jf}
J_f(\bm{x})=
\begin{bmatrix}
R\,r'(x_1) & -R\,r'(x_2)
\end{bmatrix},
\end{equation}
and therefore
\begin{equation}
\label{eq:vsa_def_promptness}
\rho(\bm{x})=
\|J_f(\bm{x})\|=
R\sqrt{r'(x_1)^2+r'(x_2)^2}.
\end{equation}

\begin{remark}[Different physical roles]
\label{rem:vsa_roles}
Although both $\sigma$ and $\rho$ are induced by the same nonlinear map, they encode different incremental phenomena: $\sigma$ maps external deflections to passive restoring torque ($\Delta\theta\mapsto\Delta\tau$), whereas $\rho$ maps internal actuation variations to active torque changes ($\Delta \bm{x}\mapsto\Delta\tau$).
\rev{For example, the force law $r(x)=\tfrac{1}{2}kx^2$, with $k>0$, has tangent stiffness $r'(x)=kx$ and corresponds to stored elastic energy $E(x)=\int_0^x r(s)ds=\tfrac{1}{6}kx^3$. Thus it is a hardening force law.
Note that $r$ here should not be confused with the stored elastic energy even though the same quadratic expression would represent the energy of a linear spring in a different notation.}
\end{remark}

\subsection{Strict monotonic coupling along constant-torque fibers}
\label{subsec:vsa_monotonic}
\begin{definition}[Internal co-contraction direction]
\label{def:co_contraction}
Let $\bm{x}\in\mathcal{F}_{\bar\tau}$.
A nonzero displacement $dx=[dx_1,dx_2]^\top\in T_{\bm{x}}\mathcal{F}_{\bar\tau}$ is an internal co-contraction if $dx_1>0$ and $dx_2>0$.
\end{definition}
\begin{proposition}[Co-contraction increases stiffness and promptness]
\label{prop:vpa}
Under Assumption~\ref{ass:vsa_hardening}, along any constant-torque fiber $\mathcal{F}_{\bar\tau}$, every internal co-contraction displacement strictly increases both $\sigma$ and $\rho$.
Moreover, on any connected $1$D branch of $\mathcal{F}_{\bar\tau}$, $\sigma$ and $\rho$ are strictly monotonically related.
\end{proposition}
\begin{proof}
Let $\bm{x}\in\mathcal{F}_{\bar\tau}$ and $dx\in T_{\bm{x}}\mathcal{F}_{\bar\tau}$.
Differentiating \eqref{eq:vsa_task_map} and enforcing $df=0$ yields
\begin{equation}
\label{eq:vsa_fiber_tangent}
0=df=R\big(r'(x_1)\,dx_1-r'(x_2)\,dx_2\big)
\Rightarrow
\tfrac{dx_2}{dx_1}=\tfrac{r'(x_1)}{r'(x_2)}.
\end{equation}
Since $r'(x)>0$ for $x>0$, \eqref{eq:vsa_fiber_tangent} implies that $dx_1$ and $dx_2$ have the same sign, hence co-contraction is an admissible fiber direction.
From \eqref{eq:vsa_def_stiffness},
\begin{equation}
\label{eq:vsa_diff_sigma}
d\sigma=R^2\big(r''(x_1)\,dx_1+r''(x_2)\,dx_2\big),
\end{equation}
which is strictly positive under Assumption~\ref{ass:vsa_hardening} and $dx_1,dx_2>0$.
Similarly, from \eqref{eq:vsa_def_promptness} and $\rho>0$,
\begin{equation}
\label{eq:vsa_diff_rho}
d\rho=
\tfrac{R}{\rho(\bm{x})}
\Big(r'(x_1)r''(x_1)\,dx_1+r'(x_2)r''(x_2)\,dx_2\Big),
\end{equation}
which is strictly positive because $r'(x_i)>0$, $r''(x_i)>0$, and $dx_i>0$.
Hence both $\sigma$ and $\rho$ strictly increase along any $\mathcal{C}^1$ co-contraction parameterization $\bm{x}(s)\subset\mathcal{F}_{\bar\tau}$.
Finally, on any connected $1$D fiber branch, strict monotonicity implies that $s\mapsto\sigma(\bm{x}(s))$ is invertible, and therefore $\rho$ is a strictly increasing function of $\sigma$ along that branch.
\end{proof}
Proposition~\ref{prop:vpa} formalizes that, for antagonistic mechanisms with strictly hardening elasticity, stiffness modulation through co-contraction is accompanied by a simultaneous and monotone increase of local torque promptness along constant-torque fibers.
\rev{This observation provides the mechanical reference point for the aerodynamic result: the previous work~\cite{Franchi2026ICUAS} used this analogy on the active promptness side, while the present work completes it on the passive impedance side.}

\section{The Dual-Rotor Aerodynamic Actuation Module}
\label{sec:dualrotor}
We introduce a minimal redundant aerodynamic actuation module that exhibits an antagonistic internal mode.
A rigid body is constrained to translate along a line and is actuated by two propellers whose thrust axes are collinear with the translation direction (Fig.~\ref{fig:dualrotor_schematic}).
Let $\nu\in\mathbb{R}$ denote the air-relative translational velocity along this axis.
The control inputs are the propeller angular speeds $\bm{v}=(v_1,v_2)\in\mathcal{V}\subset\mathbb{R}^2_{>0}$.
The propellers are mounted antagonistically so that rotor~1 contributes force in the $+\nu$ direction while rotor~2 contributes force in the $-\nu$ direction.
As a result, for a given commanded net force there typically exists a one-dimensional family of speed pairs, enabling internal co-contraction, i.e., simultaneous increase of both speeds, as in antagonistic tendon actuation~\cite{Franchi2026ICUAS}.
\rev{We use the term actuation module deliberately: the functional object is the pair of opposed rotors together with its internal allocation, not a single propeller considered in isolation.}

\subsubsection*{Non-interference assumption}
To isolate the operating-point dependence of each rotor on axial inflow, we assume negligible mutual wake interaction.
\begin{assumption}[Negligible mutual aerodynamic interference]
\label{ass:no_interference}
The rotors are mounted with sufficient separation and/or lateral offset such that, in the operating regime of interest, the inflow at each rotor disk is well approximated by the ambient axial inflow induced by the body's motion, and not by the other rotor's slipstream.
\end{assumption}
\begin{remark}[\rev{Scope of the non-interference assumption}]
\label{rem:wake_scope}
\rev{Assumption~\ref{ass:no_interference} is not intended to describe arbitrary multirotor layouts. It isolates the mechanism in a configuration where the inflow at each rotor is dominated by the body-relative axial flow. In coupled wake fields, the definition of incremental damping below still applies to the full force map, but the monotonicity property must be re-established for that coupled aerodynamic model.}
\end{remark}

\subsubsection*{Single-rotor thrust map}
Let
\[
T:\ \mathbb{R}_{\ge 0}\times \mathbb{R} \to \mathbb{R}_{\ge 0},
\qquad (v,\nu_{in})\mapsto T(v,\nu_{in})
\]
denote the thrust magnitude generated by a single propeller at angular speed $v$ under axial inflow velocity $\nu_{in}$.
We adopt the convention that $\nu_{in}>0$ denotes inflow that reduces the effective angle of attack, so that increasing $\nu_{in}$ tends to reduce thrust in the regime of interest.
Under Assumption~\ref{ass:no_interference} and the antagonistic mounting of Fig.~\ref{fig:dualrotor_schematic}, the two rotors experience opposite inflows:
\begin{equation}
\label{eq:inflows}
\nu_{in,1}=\nu,
\qquad
\nu_{in,2}=-\nu.
\end{equation}
Accordingly, the net aerodynamic force along the axis is
\begin{equation}
\label{eq:net_force}
F(\bm{v},\nu)=T(v_1,\nu)-T(v_2,-\nu).
\end{equation}

\subsubsection*{Regularity and monotonicity properties} We use the following local properties, which are standard under low-to-moderate advance-ratio operation and are supported both analytically, via the first-order Blade Element Theory model of Sec.~\ref{sec:bet}, and experimentally, through the propeller-data assessment of Sec.~\ref{sec:data_based_scale}.

\begin{assumption}[Monotone thrust in rotor speed]
\label{ass:monotone_thrust}
In the operating regime of interest, $T$ is $\mathcal{C}^2$ and $\tfrac{\partial T}{\partial v}(v,\nu_{in})>0$ for all $v>0$.
\end{assumption}
\begin{assumption}[Aerodynamic damping in inflow]
\label{ass:aero_damping}
For any fixed $v>0$, increasing axial inflow reduces thrust.
Equivalently, the inflow-sensitivity
\begin{equation}
\label{eq:def_lambda}
\lambda(v,\nu_{in})\coloneqq
-\tfrac{\partial T}{\partial \nu_{in}}(v,\nu_{in})>0
\end{equation}
holds in the operating regime.
\end{assumption}
\begin{assumption}[Aerodynamic hardening of damping]
\label{ass:aero_hardening}
The inflow-sensitivity increases with rotor speed, i.e.,
\begin{equation}
\label{eq:aero_hardening}
\tfrac{\partial \lambda}{\partial v}(v,\nu_{in})
=
-\tfrac{\partial^2 T}{\partial \nu_{in}\partial v}(v,\nu_{in})>0,
\qquad \forall v>0.
\end{equation}
\end{assumption}
\begin{remark}[Analogy with mechanical hardening]
Assumption~\ref{ass:aero_hardening} mirrors Assumption~\ref{ass:vsa_hardening}: $r''(x)>0$ implies that tangent stiffness grows with internal displacement, whereas $-\partial^2T/(\partial \nu_{in}\partial v)>0$ implies that the incremental inflow sensitivity, and thus the local damping effect defined next, grows with internal rotor speed.
Section~\ref{sec:bet} provides a first-principles verification under a simple thrust model.
\end{remark}

\section{Variable Aerodynamic Damping \rev{Actuation} (VADA)}
\label{sec:vada}
We now formalize the passive dynamical quantity that, for the dual-rotor aerodynamic module, plays the structural role corresponding to stiffness in antagonistic VSAs.
\rev{Here passive refers to the local velocity-to-force response around an already powered trim. It does not imply zero energetic cost, since maintaining the trim requires continuous rotor power, analogously to what happens in VSAs.}

\subsection{Incremental aerodynamic damping at $\nu=0$}
We define the incremental aerodynamic damping as the local resistance of the net aerodynamic force to small perturbations of air-relative velocity, evaluated at $\nu=0$:
\begin{equation}
\label{eq:def_sigma_a}
\sigma_a(\bm{v})
\coloneqq
-\left.\tfrac{\partial F}{\partial \nu}\right|_{\nu=0}.
\end{equation}
Using \eqref{eq:net_force} and the chain rule with \eqref{eq:inflows} yields
\begin{align}
\sigma_a(\bm{v})
&=
-\left[
\tfrac{\partial T}{\partial \nu_{in}}(v_1,0)
-
\tfrac{\partial T}{\partial \nu_{in}}(v_2,0)(-1)
\right]
\nonumber\\
&=\lambda(v_1,0)+\lambda(v_2,0),
\label{eq:aero_damping_sum}
\end{align}
which is structurally identical to the VSA stiffness $\sigma(\bm{x})=R^2(r'(x_1)+r'(x_2))$ in \eqref{eq:vsa_def_stiffness}, with $\lambda(\cdot,0)$ replacing $r'(\cdot)$.

\subsection{Aerodynamic co-contraction along constant-force fibers}
At $\nu=0$, define the force task map $f_a:\mathcal{V}\to\mathbb{R}$ by
\begin{equation}
\label{eq:fa_def}
f_a(\bm{v})\coloneqq F(\bm{v},0)=T(v_1,0)-T(v_2,0),
\end{equation}
and let the fiber at a commanded force $\bar F$ be
\begin{equation}
\label{eq:fiber_F}
\mathcal{F}_{\bar F}\coloneqq \{\bm{v}\in\mathcal{V}\mid f_a(\bm{v})=\bar F\}.
\end{equation}
\begin{definition}[Aerodynamic internal co-contraction]
\label{def:aero_cocontraction}
Let $\bm{v}\in\mathcal{F}_{\bar F}$.
A nonzero displacement $dv=[dv_1,dv_2]^\top\in T_{\bm{v}}\mathcal{F}_{\bar F}$ is an aerodynamic co-contraction direction if $dv_1>0$ and $dv_2>0$.
\end{definition}
\begin{proposition}[Co-contraction increases aerodynamic damping]
\label{prop:vada_damping}
Under Assumptions~\ref{ass:monotone_thrust} and \ref{ass:aero_hardening}, along any constant-force fiber $\mathcal{F}_{\bar F}$, every aerodynamic co-contraction displacement strictly increases the incremental aerodynamic damping $\sigma_a$.
\end{proposition}
\begin{proof}
Along $\mathcal{F}_{\bar F}$ we have $df_a=0$. Differentiating \eqref{eq:fa_def} gives
\begin{equation}
\label{eq:dfa0}
0=df_a=
\tfrac{\partial T}{\partial v}(v_1,0)\,dv_1
-
\tfrac{\partial T}{\partial v}(v_2,0)\,dv_2,
\end{equation}
hence
\begin{equation}
\label{eq:dv_ratio}
\tfrac{dv_2}{dv_1}
=
\tfrac{\partial T}{\partial v}(v_1,0) \,\big/\, \tfrac{\partial T}{\partial v}(v_2,0).
\end{equation}
By Assum.~\ref{ass:monotone_thrust}, the numerator and denominator are positive for $v_1,v_2>0$, hence $dv_1$ and $dv_2$ share the same sign and co-contraction is an admissible fiber direction.
From \eqref{eq:aero_damping_sum},
\begin{equation}
\label{eq:dsigma_a}
d\sigma_a=
\tfrac{\partial \lambda}{\partial v}(v_1,0)\,dv_1
+
\tfrac{\partial \lambda}{\partial v}(v_2,0)\,dv_2.
\end{equation}
Assumption~\ref{ass:aero_hardening} yields $\partial\lambda/\partial v>0$, and co-contraction implies $dv_1,dv_2>0$, hence $d\sigma_a>0$.
\end{proof}

\subsection{Active aerodynamic promptness}
\label{subsec:aero_promptness}
A complementary active quantity  \rev{introduced in~\cite{Franchi2026ICUAS}} is the local fiber density,  called aerodynamic promptness, which measures how effectively bounded rotor-speed variations generate force-rate variations.
For the scalar task map \eqref{eq:fa_def},
\begin{equation}
\label{eq:J_fa}
J_{f_a}(\bm{v})=\nabla f_a(\bm{v})=
\begin{bmatrix}
\tfrac{\partial T}{\partial v}(v_1,0) & -\tfrac{\partial T}{\partial v}(v_2,0)
\end{bmatrix},
\end{equation}
and, under the Euclidean metric on actuator rates,
\begin{equation}
\label{eq:rho_a_general}
\rho_a(\bm{v})=
\sqrt{\det(J_{f_a}J_{f_a}^\top)}=
\|J_{f_a}(\bm{v})\|.
\end{equation}
For the quadratic static thrust law $T(v_i,0)=k_Tv_i^2$ used in~\cite{Franchi2026ICUAS}, \eqref{eq:rho_a_general} yields
\begin{equation}
\rho_a(\bm{v})=2k_T\sqrt{v_1^2+v_2^2},
\end{equation}
recovering the active promptness increase under antagonistic co-contraction previously shown in~\cite{Franchi2026ICUAS}.

\subsection{Definition and structural analogy}
\begin{definition}[Variable Aerodynamic Damping \rev{Actuation}]
\label{def:vada}
\rev{Variable Aerodynamic Damping Actuation (VADA) denotes an antagonistic aerodynamic actuation module and internal allocation principle for which there exists a family of constant-force fibers $\{\mathcal{F}_{\bar F}\}$ on which co-contraction can strictly increase the incremental aerodynamic damping $\sigma_a$.}
\end{definition}
\begin{remark}[Structural analogy with antagonistic VSAs]
\label{rem:structural_analogy}
\rev{The correspondence with VSAs is local and structural. It concerns fiber motions, the common/differential decomposition, and the monotone modulation of an incremental impedance-like coefficient. It should not be read as a full physical equivalence between stiffness and damping mechanisms.}
The formal replacements are
\[
(\bm{x},\ r',\ r'')
\quad \longleftrightarrow \quad
(\bm{v},\ \lambda(\cdot,0),\ \partial\lambda/\partial v(\cdot,0)),
\]
with $\nu$ acting as the environmental velocity perturbation corresponding to the VSA environmental deflection $\theta$.
Under this mapping, co-contraction modulates passive stiffness in VSAs and incremental aerodynamic damping in the dual-rotor module, while simultaneously increasing the active promptness in both systems.
\end{remark}

\rev{
The correspondence is summarized in Table~\ref{tab:vsa_vada_missing_corner}.
The table separates the active task-responsiveness side from the passive interaction-shaping side.
In this organization, the active aerodynamic promptness entry corresponds to the fiber-density viewpoint introduced in~\cite{Franchi2026ICUAS}, whereas the bottom right corner is the passive aerodynamic quantity introduced, formalized, and validated in the present work.}

\begin{table}[t]
\centering
\footnotesize
\setlength{\tabcolsep}{3pt}
\renewcommand{\arraystretch}{1.25}
\caption{\rev{Structural organization of the VSA--VADA analogy. The previous work~\cite{Franchi2026ICUAS} addressed the active aerodynamic promptness corner (bottom left); the present work formalizes the missing passive aerodynamic damping corner (bottom right).}}
\label{tab:vsa_vada_missing_corner}
\begin{tabular}{|p{0.18\columnwidth}|p{0.37\columnwidth}|p{0.38\columnwidth}|}
\hline
&
\textbf{\rev{Active task responsiveness}}
&
\textbf{\rev{Passive interaction shaping}}
\\
\hline
\textbf{\rev{Antagonistic VSA/Muscles}}
&
\rev{Task/torque map $f(\bm{x})$ and torque promptness $\rho(\bm{x})$}
&
\rev{Joint stiffness
$\sigma(\bm{x})=-\left.\partial \tau/\partial \theta\right|_{\theta=0}$}
\\
\hline
\textbf{\rev{Dual-rotor module}}
&
\rev{Force map $f_a(\bm{v})$ and aerodynamic promptness $\rho_a(\bm{v})$}
&
\rev{Aerodynamic damping
$\sigma_a(\bm{v})=-\left.\partial F/\partial \nu\right|_{\nu=0}$}
\\
\hline
\end{tabular}
\end{table}

\section{Validation via Blade Element Theory}
\label{sec:bet}
This section validates \rev{theoretically} Assumptions~\ref{ass:aero_damping}--\ref{ass:aero_hardening} by exhibiting a minimal thrust expression derived from a first-order Blade Element Theory (BET) argument~\cite{volodin2022blade}.
\rev{The derivation is not used as a high-fidelity rotorcraft model. It is used to expose, in closed form, the speed--inflow coupling that makes the damping coefficient grow with rotor speed.}

\subsection{A first-order thrust model affine in inflow}
To avoid notational overload, we consider a single propeller and denote its angular speed by $v>0$.
Let $\nu_{in}\in\mathbb{R}$ denote the axial inflow velocity at the rotor disk, with the sign convention of Sec.~\ref{sec:dualrotor}.
Consider a propeller with $N$ blades of radius $B$, constant chord $c$, and constant pitch angle $\theta_0$.
For a blade element at radial coordinate \rev{$\varrho\in[0,B]$,} the tangential speed is $v\varrho$.
Under a small-angle, first-order approximation and neglecting induced-inflow and profile-drag corrections, the inflow angle satisfies
\begin{equation}
\phi(\varrho)\approx \tfrac{\nu_{in}}{v\varrho},
\end{equation}
so that \rev{the angle of attack $\alpha$ is approximated by:}
\begin{equation}
\alpha(\varrho)\approx \theta_0-\phi(\varrho)
\approx \theta_0-\tfrac{\nu_{in}}{v\varrho}.
\end{equation}
Assuming a linear lift regime, \rev{the lift coefficient is}  $C_L=a\alpha$ with lift-curve slope $a>0$, and 
the \rev{elemental thrust contribution of the annular blade elements of radial width \(d\varrho\), summed over the \(N\) blades, is}
\begin{align}
dT
&\approx
\tfrac{1}{2}N\rho_a c\, C_L (v\varrho)^2 d\varrho \nonumber =
\tfrac{1}{2}N\rho_a c a\left(\theta_0-\tfrac{\nu_{in}}{v\varrho}\right)(v\varrho)^2 d\varrho \nonumber\\
&=
\tfrac{1}{2}N\rho_a c a\left(\theta_0 v^2\varrho^2-v\varrho\nu_{in}\right)d\varrho,
\end{align}
where $\rho_a$ is the air density.
Integrating over $\varrho\in[0,B]$ yields
\begin{equation}
\label{eq:bet_T}
T(v,\nu_{in})=
\tfrac{1}{2}N\rho_a c a\left(\tfrac{\theta_0B^3}{3}v^2-\tfrac{B^2}{2}v\nu_{in}\right).
\end{equation}
Defining
\begin{equation}
\label{eq:kT_kD_defs}
k_T \coloneqq \tfrac{1}{6}N\rho_a c a\theta_0B^3>0,
\quad
k_D \coloneqq \tfrac{1}{4}N\rho_a c aB^2>0,
\end{equation}
we obtain the affine-inflow thrust model
\begin{equation}
\label{eq:affine_inflow}
T(v,\nu_{in}) = k_Tv^2-k_Dv\nu_{in}.
\end{equation}

\subsection{Verification of damping and hardening}
\begin{proposition}[Verification of aerodynamic damping]
\label{prop:verify_damping}
For \eqref{eq:affine_inflow}, the inflow-sensitivity $\lambda(v,\nu_{in})\coloneqq -\partial T/\partial \nu_{in}$ satisfies $\lambda(v,\nu_{in})>0$ for all $v>0$.
\end{proposition}
\begin{proof}
From \eqref{eq:affine_inflow}, $\lambda(v,\nu_{in})=k_Dv$, which is strictly positive for $v>0$ since $k_D>0$.
\end{proof}
\begin{proposition}[Verification of aerodynamic hardening]
\label{prop:verify_hardening}
For \eqref{eq:affine_inflow}, the inflow-sensitivity is strictly increasing in rotor speed: $\tfrac{\partial \lambda}{\partial v}(v,\nu_{in})>0$ for all $v>0$.
\end{proposition}
\begin{proof}
Using $\lambda(v,\nu_{in})=k_Dv$ gives $\tfrac{\partial \lambda}{\partial v}=k_D>0$.
\end{proof}
\begin{remark}[Consistency with monotone thrust in rotor speed]
\label{rem:monotone_T}
Model \eqref{eq:affine_inflow} yields $\tfrac{\partial T}{\partial v}=2k_Tv-k_D\nu_{in}$.
Hence $\partial T/\partial v>0$ holds at least locally around $\nu_{in}=0$ for any $v>0$, and more generally whenever $\nu_{in}<2(k_T/k_D)v$.
This supports Assumption~\ref{ass:monotone_thrust} in the operating regime used in Sec.~\ref{sec:vada}.
\end{remark}

\subsection{A minimal quadratic-force analogue}
The mixed (multi-linear) term $-k_Dv\nu_{in}$ in \eqref{eq:affine_inflow} is the ingredient that enables VADA behavior.
It implies $\lambda(v,\nu_{in})=k_Dv$, i.e., the incremental inflow sensitivity, and thus the incremental damping defined in Sec.~\ref{sec:vada}, grows linearly with rotor speed.
This is structurally analogous to the VSA force law $r(x)=\tfrac{1}{2}kx^2$, for which the tangent stiffness $r'(x)=kx$ grows linearly with internal displacement.
Accordingly, within the validity of the first-order model, aerodynamic co-contraction provides a mechanism to tune an effective aerodynamic damping around a trim, complementing the active promptness viewpoint of~\cite{Franchi2026ICUAS}, \rev{see Table~\ref{tab:vsa_vada_missing_corner}}.

\section{Generalization to a Generic Trim Airspeed}
\label{sec:trim_general}
Sections~\ref{sec:dualrotor}--\ref{sec:vada} specialized the definitions to $\nu=0$, which is the natural analogue of the VSA equilibrium $\theta=0$.
The same incremental notion of aerodynamic damping, however, is defined at any trim airspeed $\bar\nu\in\mathbb{R}$ and can be modulated along the corresponding constant-force fibers.

\subsection{Incremental damping around $\nu=\bar\nu$}
Recall the net force map $F(\bm{v},\nu)=T(v_1,\nu)-T(v_2,-\nu)$, with $\bm{v}=(v_1,v_2)\in\mathcal{V}\subset\mathbb{R}^2_{>0}$.
For a fixed operating point $\bm{v}$ and trim airspeed $\bar\nu$, define the incremental aerodynamic damping as
\begin{equation}
\label{eq:def_sigma_a_bar}
\sigma_a(\bm{v};\bar\nu)
\coloneqq
-\left.\tfrac{\partial F}{\partial \nu}\right|_{\nu=\bar\nu}.
\end{equation}
Using the chain rule and $\lambda(v,\nu_{in})\coloneqq-\partial T/\partial \nu_{in}$ yields
\begin{equation}
\label{eq:sigma_a_bar}
\sigma_a(\bm{v};\bar\nu)=
\lambda(v_1,\bar\nu)+\lambda(v_2,-\bar\nu).
\end{equation}
Under Assumption~\ref{ass:aero_damping}, $\sigma_a(\bm{v};\bar\nu)>0$ in the operating regime of interest.
This is the direct trim analogue of the VSA stiffness $\sigma(\bm{x})$ defined at $\theta=0$ in \eqref{eq:vsa_def_stiffness}.

\subsection{Co-contraction on constant-force fibers at $\nu=\bar\nu$}
Fix a desired net force $\bar F$ at the trim $\bar\nu$, and define the trim task map $f_{a,\bar\nu}:\mathcal{V}\to\mathbb{R}$ by
\begin{equation}
\label{eq:fa_bar}
f_{a,\bar\nu}(\bm{v})\coloneqq F(\bm{v},\bar\nu)=T(v_1,\bar\nu)-T(v_2,-\bar\nu).
\end{equation}
The associated fiber is
\begin{equation}
\label{eq:fiber_bar}
\mathcal{F}_{\bar F}^{\bar\nu}
\coloneqq
\left\{\bm{v}\in\mathcal{V}\ \middle|\ f_{a,\bar\nu}(\bm{v})=\bar F\right\}.
\end{equation}
\begin{proposition}[Co-contraction increases incremental damping at a trim]
\label{prop:vada_trim}
Assume that $T$ is $\mathcal{C}^2$ and that Assumptions~\ref{ass:monotone_thrust} and \ref{ass:aero_hardening} hold at inflows $\nu_{in}\in\{\bar\nu,-\bar\nu\}$.
Then, along any constant-force fiber $\mathcal{F}_{\bar F}^{\bar\nu}$, every aerodynamic co-contraction displacement strictly increases $\sigma_a(\cdot;\bar\nu)$.
\end{proposition}
\begin{proof}
\rev{The proof is the trim-shifted version of Proposition~\ref{prop:vada_damping}. Differentiating $f_{a,\bar\nu}$ gives \[ 0= \tfrac{\partial T}{\partial v}(v_1,\bar\nu)\,dv_1 - \tfrac{\partial T}{\partial v}(v_2,-\bar\nu)\,dv_2 . \] By Assumption~\ref{ass:monotone_thrust}, aerodynamic co-contraction is an admissible fiber direction. Differentiating \eqref{eq:sigma_a_bar} gives \[ d\sigma_a= \tfrac{\partial\lambda}{\partial v}(v_1,\bar\nu)\,dv_1+ \tfrac{\partial\lambda}{\partial v}(v_2,-\bar\nu)\,dv_2 , \] which is strictly positive by Assumption~\ref{ass:aero_hardening}. }
\end{proof}

\begin{remark}[Wind and air-relative speed]
If the body speed is $v_{\text{\rm body}}$ and the wind speed along the axis is $v_{\text{\rm wind}}$, then $\nu=v_{\text{\rm body}}-v_{\text{\rm wind}}$.
The above trim formulation therefore applies under steady wind by selecting $\bar\nu=\bar v_{\text{\rm body}}-v_{\text{\rm wind}}$.
\end{remark}

\section{Aerodynamic Impedance and Passive Velocity Control with Variable Damping}
\label{sec:impedance_control}
The physical meaning of the incremental damping becomes transparent by embedding the force map into the translational dynamics.
Assume still air, let $m>0$ be the body mass, $\nu\in\mathbb{R}$ the air-relative velocity, and $F_{\text{\rm ext}}$ an external force along the axis.
The dynamics read
\begin{equation}
\label{eq:newton_general}
m\dot{\nu}=F(\bm{v},\nu)+F_{\text{\rm ext}},
\end{equation}
with $F(\bm{v},\nu)$ given by \eqref{eq:net_force}.

\subsection{Impedance structure under the affine-inflow model}
Using \eqref{eq:affine_inflow} in \eqref{eq:net_force} with $\nu_{in,1}=\nu$ and $\nu_{in,2}=-\nu$, and substituting into \eqref{eq:newton_general}, gives
\begin{equation}
\label{eq:cart_dynamics}
m\dot{\nu}=k_T(v_1^2-v_2^2)-k_D(v_1+v_2)\nu+F_{\text{\rm ext}}.
\end{equation}
Rearranging,
\begin{equation}
\label{eq:impedance_form}
m\dot{\nu}+c_{\text{\rm app}}(\bm{v})\nu
=F_{\text{\rm act}}(\bm{v})+F_{\text{\rm ext}},
\end{equation}
where
\begin{equation}
\label{eq:definitions_capp_fact}
c_{\text{\rm app}}(\bm{v})\coloneqq k_D(v_1+v_2),
\qquad
F_{\text{\rm act}}(\bm{v})\coloneqq k_T(v_1^2-v_2^2).
\end{equation}
The coefficient $c_{\text{\rm app}}(\bm{v})$ is a physical viscous term in this model: it multiplies the air-relative velocity and therefore represents a local velocity-to-force \rev{impedance}.
Moreover, since $\lambda(v,0)=k_Dv$, one has $c_{\text{\rm app}}(\bm{v})=\sigma_a(\bm{v})$ as defined in \eqref{eq:aero_damping_sum}.

\subsection{Equilibrium velocity and common/differential-mode roles}
Set $F_{\text{\rm ext}}=0$ and consider constant rotor speeds.
An equilibrium velocity $\nu_{\text{\rm eq}}$ satisfies
\begin{equation}
\label{eq:eq_balance}
k_D(v_1+v_2)\nu_{\text{\rm eq}}
=
k_T(v_1^2-v_2^2)
=
k_T(v_1-v_2)(v_1+v_2).
\end{equation}
Provided $v_1+v_2>0$,
\begin{equation}
\label{eq:nu_eq}
\nu_{\text{\rm eq}}(\bm{v})=\tfrac{k_T}{k_D}(v_1-v_2).
\end{equation}
Equation~\eqref{eq:impedance_form} can then be written as
\begin{equation}
\label{eq:impedance_shifted}
m\dot{\nu}+c_{\text{\rm app}}(\bm{v})\big(\nu-\nu_{\text{\rm eq}}(\bm{v})\big)=F_{\text{\rm ext}}.
\end{equation}
Thus, in the affine-inflow regime, the differential component $(v_1-v_2)$ sets the equilibrium air-relative velocity, while the common-mode component $(v_1+v_2)$ sets the local damping.
\rev{At a fixed trim, force tracking defines a fiber, while the internal coordinate can be selected to trade power, damping, and active promptness. One possible allocation interpretation is to choose $\bm{v}$ on the force fiber according to a supervisory objective combining these quantities, e.g.,}
\begin{equation}
\label{eq:allocation_interpretation}
\rev{
\min_{\bm{v}\in\mathcal{V}}
\{P(\bm{v})-\alpha_P\sigma_a(\bm{v};\bar\nu)-\beta_P\rho_a(\bm{v})\}
\quad
\text{s.t.}\quad
F(\bm{v},\bar\nu)=\bar F,}
\end{equation}
\rev{where $P$ is a power proxy and $\alpha_P,\beta_P\ge 0$ encode the desired readiness level. The present work does not propose a full controller; it identifies the passive damping quantity that such an allocator could regulate.}

\rev{
\section{Propeller-Data-Based Scale Assessment} \label{sec:data_based_scale} 
}

\rev{
The analysis above shows that the VADA mechanism is governed by the local
inflow-sensitivity coefficient $\lambda(v,\nu_{in})$ defined in
\eqref{eq:def_lambda}. We now estimate the corresponding force--velocity
derivative from published propeller measurements. The purpose is to test
whether the derivative used in the theory has a non-negligible magnitude for
small-UAV-scale propellers, and to compare its scale with ordinary body-drag
damping.

We use the UIUC Propeller Database~\cite{UIUCPropDB}, which reports fixed-RPM
wind-tunnel measurements of propeller thrust and power while sweeping the
axial freestream velocity. This data structure is directly suited to our
purpose, because it provides the local variation of thrust with axial inflow
at prescribed propeller speeds.
The database is expressed in the standard whole-propeller coefficients
\begin{equation}
T = \rho_a n^2D^4 C_T(J), \qquad
P = \rho_a n^3D^5 C_P(J),
\label{eq:uiuc_coefficients}
\end{equation}
where \(T\) is the thrust force, \(P\) is shaft power, \(C_T\) and \(C_P\)
are the nondimensional thrust and power coefficients, \(n=v/(2\pi)\) is the
rotational frequency in revolutions per second, \(D\) is the propeller
diameter, and
\begin{equation}
J=\tfrac{\nu_{in}}{nD}
\label{eq:advance_ratio_database}
\end{equation}
is the advance ratio. These whole-propeller coefficients are distinct from
the sectional lift coefficient \(C_L\) used in Sec.~\ref{sec:bet}.

Differentiating \eqref{eq:uiuc_coefficients} with respect to the physical
axial inflow $\nu_{in}$ at fixed $v$, and using
\eqref{eq:advance_ratio_database}, gives
\begin{equation}
\tfrac{\partial T}{\partial \nu_{in}} = \rho_a nD^3 C_T'(J).
\end{equation}
Consequently, the single-rotor coefficient $\lambda(v,0)$ in
\eqref{eq:def_lambda} can be obtained directly from the low-$J$ slope of the
measured thrust-coefficient curve:
\begin{equation}
\lambda(v,0)
=
-\left.\tfrac{\partial T}{\partial \nu_{in}}\right|_{\nu_{in}=0}
=
\rho_a nD^3\bigl(-C_T'(0)\bigr).
\label{eq:lambda_from_CT_slope}
\end{equation}
For an antagonistic pair of identical propellers operated on the symmetric
co-contraction branch, \eqref{eq:aero_damping_sum} therefore becomes
\begin{equation}
\sigma_a(v_b,v_b)
=
2\rho_a nD^3\bigl(-C_T'(0)\bigr),
\qquad
n=\tfrac{v_b}{2\pi}.
\label{eq:data_pair_damping}
\end{equation}

For each fixed-RPM run, we estimate $C_T'(0)$ by fitting
\begin{equation}
C_T(J)\simeq C_T(0)+C_T'(0)J
\label{eq:CT_linear_fit}
\end{equation}
using the static value $C_T(0)$, interpolated from the corresponding UIUC
static file, together with the first low-advance-ratio points satisfying
$J\leq 0.25$. A run is used only if at least two dynamic points are available
in this low-$J$ interval, the fitted slope satisfies $-C_T'(0)>0$, and the
coefficient of determination of the local linear fit satisfies
\begin{equation}
R_{\rm fit}^2 \geq 0.85.
\label{eq:fit_quality_condition}
\end{equation}
The last condition is used only to exclude cases where the available low-$J$
samples do not represent a sufficiently local linear trend. No parameter of
the VADA model is tuned to match the data. We use
$\rho_a=1.225\,\mathrm{kg/m^3}$ and the nominal diameters encoded in the UIUC
file names. The corresponding pair aerodynamic power is computed from the
static power coefficient as
\begin{equation}
P_{\mathrm{pair}} = 2\rho_a n^3D^5C_P(0).
\label{eq:data_pair_power}
\end{equation}

To avoid relying on a single propeller, we consider three characteristic UIUC
families for which at least four propellers are available in the selected
data: APC Slow Flyer, APC Sport, and APC Thin Electric. 
These families provide 16 propellers and 62 retained low-$J$ runs after excluding cases with fewer than two retained fixed-RPM runs.
The family-level results are
summarized in Table~\ref{tab:family_summary}. The data-derived damping is
consistently of order $10^{-1}\,\mathrm{N\,s/m}$ and reaches values close to
$1\,\mathrm{N\,s/m}$ for the most favorable slow-flyer geometries.
}

\begin{table}[t]
\centering
\caption{\rev{Family-level summary of the data-derived VADA scale.}}
\label{tab:family_summary}
\scriptsize
\setlength{\tabcolsep}{4.0pt}
\renewcommand{\arraystretch}{1.12}
\rev{
\begin{tabular}{lccccc}
\toprule
APC Family & Props. & Runs &
$\sigma_a$ range &
Median $\sigma_a$ &
Median $P_{\mathrm{pair}}$ \\
& & & [Ns/m] & [Ns/m] & [W] \\
\midrule
Slow Flyer    & 6  & 24 & 0.194--0.855 & 0.340 & 75.2 \\
Sport         & 4  & 16 & 0.109--0.471 & 0.289 & 86.5 \\
Thin Electric & 6  & 22 & 0.121--0.404 & 0.243 & 75.5 \\
\midrule
All           & 16 & 62 & 0.109--0.855 & 0.296 & 75.5 \\
\bottomrule
\end{tabular}}
\end{table}

\rev{
A propeller-level summary is reported in Table~\ref{tab:propeller_summary}.
The variation across families is physically meaningful. The APC Slow Flyer
propellers, which are intended for low-speed operation, display the largest
low-$J$ thrust sensitivities in the analyzed set. The APC Sport and APC Thin
Electric families still produce positive values of $-C_T'(0)$ in the retained
runs, but their effective coefficients are generally smaller. Thus, the data
support the mechanism while also showing that VADA performance is
propeller-dependent: co-contraction can shape passive aerodynamic damping,
but effective use of the principle requires propellers with a sufficiently
large low-advance-ratio thrust sensitivity.
}

\begin{table}[t]
\centering
\caption{\rev{Propeller-level data-derived VADA scale. The damping range is the antagonistic-pair value \eqref{eq:data_pair_damping}. The power range is computed from \eqref{eq:data_pair_power}. The last column reports $m/\sigma_a$ for $m=0.5\,\mathrm{kg}$ in the same order as the damping range, so that larger damping corresponds to a smaller velocity-decay time.}}
\label{tab:propeller_summary}
\scriptsize
\setlength{\tabcolsep}{3.2pt}
\renewcommand{\arraystretch}{1.08}
\rev{
\begin{tabular}{llccccc}
\toprule
APC Family & Propeller & Runs & RPM range & $\sigma_a$ & $P_{\mathrm{pair}}$ & $m/\sigma_a$ \\
& & & & [Ns/m] & [W] & [s] \\
\midrule
Slow Flyer & 9$\times$4.7  & 4 & 4008--6815 & 0.235--0.475 & 22--108  & 2.13--1.05 \\
Slow Flyer & 9$\times$6    & 4 & 4016--6523 & 0.194--0.229 & 34--167  & 2.58--2.19 \\
Slow Flyer & 10$\times$4.7 & 4 & 4014--6512 & 0.333--0.660 & 36--176  & 1.50--0.76 \\
Slow Flyer & 10$\times$7   & 4 & 3008--6006 & 0.217--0.347 & 22--207  & 2.30--1.44 \\
Slow Flyer & 11$\times$3.8 & 4 & 3001--6011 & 0.311--0.713 & 15--141  & 1.61--0.70 \\
Slow Flyer & 11$\times$4.7 & 4 & 3004--6006 & 0.329--0.855 & 21--206  & 1.52--0.58 \\
\midrule
Sport & 10$\times$5 & 4 & 4005--6301 & 0.270--0.359 & 34--125 & 1.85--1.39 \\
Sport & 10$\times$6 & 4 & 4006--6496 & 0.109--0.296 & 41--162 & 4.59--1.69 \\
Sport & 11$\times$5 & 4 & 3009--6002 & 0.281--0.471 & 22--162 & 1.78--1.06 \\
Sport & 11$\times$6 & 4 & 3012--5992 & 0.180--0.362 & 27--196 & 2.78--1.38 \\
\midrule
Thin Electric & 9$\times$4.5  & 4 & 4002--6917 & 0.241--0.295 & 17--92  & 2.08--1.69 \\
Thin Electric & 9$\times$6    & 4 & 4003--6701 & 0.121--0.201 & 24--110 & 4.12--2.49 \\
Thin Electric & 10$\times$5   & 4 & 4005--6707 & 0.205--0.297 & 28--137 & 2.44--1.69 \\
Thin Electric & 10$\times$7   & 2 & 6020--6531 & 0.209--0.236 & 134--171 & 2.40--2.12 \\
Thin Electric & 11$\times$5.5 & 4 & 3010--6002 & 0.350--0.404 & 16--134 & 1.43--1.24 \\
Thin Electric & 11$\times$7   & 4 & 3003--5988 & 0.126--0.291 & 23--186 & 3.97--1.72 \\
\bottomrule
\end{tabular}
}
\end{table}

\rev{
The identified values should be compared with the incremental damping created by ordinary body drag. For the standard quadratic model \begin{equation} F_{\mathrm{drag}}(\nu) = \tfrac{1}{2}\rho_a C_DA\,\nu|\nu|, \end{equation} the local damping about a trim air-relative speed $\bar\nu$ is \begin{equation} c_b(\bar\nu) = \left. \tfrac{\partial F_{\mathrm{drag}}}{\partial \nu} \right|_{\nu=\bar\nu} = \rho_a C_DA|\bar\nu|. \label{eq:body_drag_damping} \end{equation} This term vanishes at hover, whereas $\sigma_a$ in \eqref{eq:data_pair_damping} remains finite around a powered co-contracted trim. Table~\ref{tab:body_drag_comparison} reports representative values for $C_DA\in[0.01,0.05]\,\mathrm{m^2}$, compatible with small multirotor projected drag areas. At $|\bar\nu|=5\,\mathrm{m/s}$, the resulting body-drag damping is $0.061$--$0.306\,\mathrm{N\,s/m}$, while the median UIUC data-derived VADA damping over the selected families is about $0.296\,\mathrm{N\,s/m}$ and the largest values exceed $0.8\,\mathrm{N\,s/m}$. Thus, the identified effect is of the same scale as ordinary low-speed aerodynamic damping and is finite even at hover, where the quadratic body-drag linearization vanishes.
}

\begin{table}[t] \centering \caption{\rev{Ordinary body-drag linearization from \eqref{eq:body_drag_damping}, using $\rho_a=1.225\,\mathrm{kg/m^3}$.}} \label{tab:body_drag_comparison} \scriptsize \setlength{\tabcolsep}{4pt} \renewcommand{\arraystretch}{1.12} 
\rev{\begin{tabular}{ccc} \toprule $|\bar\nu|$ [m/s] & $c_b$ for $C_DA=0.01$ [Ns/m] & $c_b$ for $C_DA=0.05$ [Ns/m] \\ \midrule 0 & 0 & 0 \\ 2 & 0.025 & 0.123 \\ 5 & 0.061 & 0.306 \\ 10 & 0.123 & 0.613 \\ 15 & 0.184 & 0.919 \\ \bottomrule \end{tabular}}
\end{table}

\rev{
\section{Critical Discussion and Future Work}
\label{sec:discussion}
}

\rev{
The preceding sections establish VADA at three levels:
(i) internal fiber motions in antagonistic actuation,
(ii) the BET-based speed--inflow coupling that makes damping grow with rotor speed, and
(iii) propeller-data evidence that the resulting scale is relevant for small-UAV propellers.
The UIUC data support the central modeling assumption: real propellers exhibit positive low-inflow thrust sensitivity, yielding antagonistic-pair damping of order $10^{-1}\,\mathrm{N\,s/m}$, with favorable cases approaching $1\,\mathrm{N\,s/m}$.

\subsubsection*{Structural analogy and aerodynamic meaning}
The VSA--VADA correspondence is local and structural.
It concerns fiber motions, internal co-contraction, and modulation of an incremental impedance-like coefficient while preserving the commanded task output.
In an antagonistic VSA, co-contraction modulates the displacement-to-torque map
\(
    \Delta\theta \mapsto \Delta\tau,
\)
i.e., stiffness.
In the dual-rotor module, co-contraction modulates the air-relative-velocity-to-force map
\(
    \Delta\nu \mapsto \Delta F,
\)
i.e., aerodynamic damping.
Thus the common structure is the use of an internal redundant coordinate to change a passive local response without changing the commanded task variable.

Under quasi-steady aerodynamics, force depends on local flow conditions, including axial inflow.
The passive term appearing in a trim linearization is therefore velocity-dependent:
\(
    \sigma_a(\bm v;\bar\nu)
    =
    -\left.\tfrac{\partial F}{\partial \nu}\right|_{\nu=\bar\nu}.
\)
The internal fiber coordinate changes this local velocity-to-force slope while keeping the trim force fixed, giving aerodynamic co-contraction its passive interpretation.

\subsubsection*{Magnitude, power, and propeller dependence}
The data-based assessment in Sec.~\ref{sec:data_based_scale} shows that the identified damping is quantitatively relevant.
For the three UIUC propeller families considered, the retained low-advance-ratio data imply
\(
    \sigma_a \simeq 0.1\text{--}0.9\,\mathrm{N\,s/m}.
\)
These values are comparable to ordinary body-drag linearization at low-to-moderate air-relative speeds.
For a quadratic drag model, body-drag damping vanishes at $\bar\nu=0$ and grows linearly with trim airspeed, whereas $\sigma_a$ is finite around a powered co-contracted trim.
VADA is therefore most relevant around hover, near-hover, and low-speed operation, where passive body-drag damping is weak and disturbance rejection or wind interaction can be critical.

VADA complements aerodynamic promptness~\cite{Franchi2026ICUAS}: the same internal loading that increases active force responsiveness along a constant-force fiber can also increase the passive local velocity-to-force response.
This suggests an allocation viewpoint in which the internal coordinate regulates effort, active authority, and passive aerodynamic impedance.
The energetic cost remains central: in the affine-inflow model damping grows with common-mode rotor speed, whereas the usual aerodynamic power proxy grows cubically.
The UIUC data show the same qualitative tradeoff through non-negligible power values associated with larger identified damping.

The effect is also propeller-dependent.
Different propeller families exhibit different low-$J$ thrust slopes and therefore different effective values of $\lambda(v,0)$.
The analyzed data show that APC Slow Flyer propellers generally provide larger damping values than several APC Thin Electric and APC Sport propellers.
Thus, effective use of VADA requires propellers with sufficiently large low-advance-ratio thrust sensitivity; co-contraction supplies the allocation mechanism, while blade geometry determines the achievable damping per unit power.

\subsubsection*{Scope, limitations, and future work}
The minimal dual-rotor model isolates the mechanism under prescribed, non-interfering axial inflows.
Complete multirotor platforms introduce coupled wakes, tilted propellers, body attitude changes, rotor dynamics, nonuniform inflow, and vehicle-level aerodynamic couplings.
These effects modify the force map and may change the monotonicity or magnitude of the damping modulation.
Nevertheless, the definition of trim incremental damping applies to any differentiable aerodynamic force map.
For each architecture, the relevant task is to identify internal wrench-preserving directions that increase the chosen damping measure.

Future work will identify $\sigma_a$ directly on a dedicated dual-rotor setup and extend the analysis to complete multirotor platforms with coupled aerodynamic effects.
}

\section*{Acknowledgments}
The authors thanks colleagues, mentors, and mentees for inspiring discussions. 
The authors also acknowledges the LLM Gemini 3 (2026) and Copilot (GPT 5.5) for assistance with proofreading and code-generation. The accuracy of all resulting outputs was carefully and manually verified by the authors.

\printbibliography[title={References}]
\end{document}